\documentclass{article}

\usepackage{microtype}
\usepackage{graphicx}
\usepackage{subcaption}
\usepackage{booktabs} 

\usepackage{hyperref}

\usepackage[accepted]{icml2026}

\usepackage{amsmath}
\usepackage{amssymb}
\usepackage{mathtools}
\usepackage{amsthm}
\usepackage[table]{xcolor}
\usepackage{multirow}

\usepackage[capitalize,noabbrev]{cleveref}

\theoremstyle{plain}
\newtheorem{theorem}{Theorem}[section]

\theoremstyle{definition}

\theoremstyle{remark}

\newcommand{\RMSNorm}{\operatorname{RMSNorm}}

\usepackage[textsize=tiny]{todonotes}

\icmltitlerunning{Tensorizing Engram: Sharing Latents Across N-Gram Embeddings is Beneficial in LLMs}

\begin{document}

\twocolumn[
  \icmltitle{Tensorizing Engram: Sharing Latents Across N-Gram Embeddings is Beneficial in LLMs}

  \icmlsetsymbol{equal}{*}

\begin{icmlauthorlist}
    \icmlauthor{Wuyang Zhou}{yyy}
    \icmlauthor{Yuxuan Gu}{yyy}
    \icmlauthor{Giorgos Iacovides}{yyy}
    \icmlauthor{Yuning Qiu}{zzz}
    \icmlauthor{Qibin Zhao}{zzz}
    \icmlauthor{Danilo Mandic}{yyy}
\end{icmlauthorlist}

  \icmlaffiliation{yyy}{Department of Electrical and Electronic Engineering, Imperial College London, London, United Kingdom}
  \icmlaffiliation{zzz}{AIP, RIKEN, Tokyo, Japan}

  \icmlcorrespondingauthor{Wuyang Zhou}{wuyang.zhou19@imperial.ac.uk}
  \icmlkeywords{Machine Learning, ICML}

  \vskip 0.3in
]

\printAffiliationsAndNotice{}  

\begin{abstract}
Modern language models represent text using discrete token-level embeddings, which forces recurring multi-token
patterns to be learned implicitly across Transformer layers. Both Over-tokenized Transformers and Engram attempt
to address this limitation by explicitly incorporating multi-token (n-gram) memories. However, they rely on separate hash tables for each n-gram order, which introduces hash collisions and prevents nested n-grams from sharing the underlying latent structures. To address these issues, we propose \underline{\textbf{T}}ensorized E\underline{\textbf{ngram}} (TN-gram), a compact memory module that represents tensorized $n$-gram embeddings through shared factors in the Canonical Polyadic (CP) form. TN-gram learns shared token-position factors together with order-absorption vectors to encode the embeddings of different $n$-gram order. Comprehensive experiments demonstrate that TN-gram matches or even outperforms Engram-style $n$-gram modules while requiring much fewer parameters. 
\end{abstract}

\section{Introduction}

Modern Language Models (LMs) \cite{achiam2023gpt,grattafiori2024llama} represents text primarily through token-level embeddings \cite{brown2020language, devlin2019bert, vaswani2017attention}. While this embedding design is simple and scalable, it requires recurring multi-token ($n$-gram) patterns, such as phrases, collocations, and syntactic fragments, to be constructed implicitly across Transformer \cite{vaswani2017attention} layers. This is inevitably inefficient, as model capacity must be split between reasoning and recovering frequent local dependencies that are naturally contiguous $n$-gram contexts, i.e., short sequences of $n$ adjacent tokens \cite{chen1999empirical, deepseek_engram, basharin2024faster, goodman2001bit, overtokenized_transformer, manning1999foundations}.

To utilize $n$-gram contexts, Over-tokenized Transformer \cite{overtokenized_transformer, meituan} expands the embedding vocabulary with higher-order $n$-gram tokens, while DeepSeek's Engram \cite{deepseek_engram} injects $n$-gram information through external memory layers, outperforming the Over-tokenized Transformer. These works have demonstrated that scaling explicit $n$-gram memories can be more effective than scaling the number of experts in the Mixture-of-Experts (MoE) layers \cite{shazeer2017,zhou2022mixture}, and can improve the language modeling performance of LLMs without incurring much more FLOPs. 

However, direct modeling of the whole $n$-gram space is prohibitive in terms of the curse of dimensionality, as $n$-gram embedding requires exponentially growing number of parameters, i.e., $V^n \times d$, where $V$ is the vocabulary size, and $d$ is the model dimension. Assume $V=30,000$ and $d=4096$, then even the simplest $2$-gram embedding already requires 3.69 trillion parameters. Thus, existing works utilizing explicit $n$-gram memories in LLMs usually resort to hash tables to obtain tractable lookup tables, but suffer from hash collisions where distinct $n$-grams share the same slot \cite{deepseek_engram, overtokenized_transformer, meituan}. Furthermore, since a $n$-gram has $V^n$ possible token combinations, if $n$ is large (e.g., $n=6$), the model might not encounter all $V^n$ combinations in training, leaving their embeddings in the hash tables completely untrained.

\begin{figure*}[t!]
  \centering
  \includegraphics[width=0.95\linewidth]{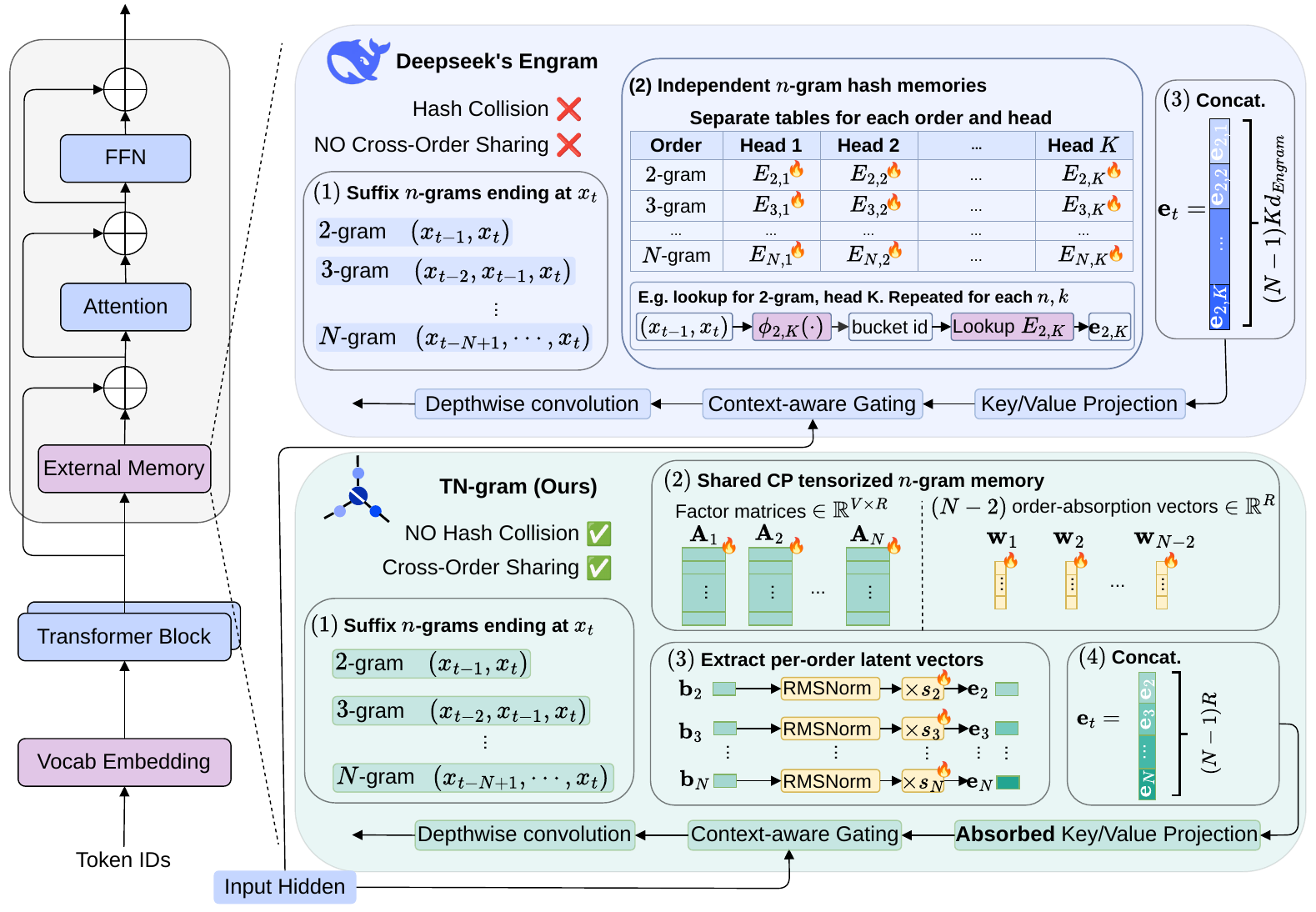}
  \caption{Architecture of Engram \cite{deepseek_engram} and the proposed TN-gram. Both methods inject explicit $N$-gram memory into selected Transformer blocks. Engram uses separate per-order, per-head hash tables $E_{n,k}$, while TN-gram replaces them with a tensorized CP memory module over shared factors ,$\mathbf{A}_1,\ldots,\mathbf{A}_N$, and order-absorption vectors, $\mathbf{w}_1,\ldots,\mathbf{w}_{N-2}$, enabling cross-order latent sharing within differernt $n$-gram orders.
}
\label{fig:architecture}
\end{figure*}

Additionally, the separate hash tables for different $n$-gram orders in existing works \cite{deepseek_engram, overtokenized_transformer, meituan} completely overlook an important intrinsic structure, i.e., an $(n-1)$-gram is always contained within its higher-order $n$-grams. For example, assume ``New'', ``York'', and ``City'' are three individual tokens and we use all three $n$-gram orders ($1$-gram, $2$-gram, and $3$-gram). Representing the phrase ``New York City'' at the last position requires the $1$-gram embedding of ``City'', the $2$-gram embedding of ``York City'', and the $3$-gram embedding of ``New York City''. This nested structure naturally suggests enforcing shared latent components across embeddings of different $n$-gram orders, which in turn enables training embeddings for higher-order $n$-grams even without encountering them in the training set. This motivates us to ask the following question:

\textit{Can we model the complete $n$-gram spaces without the drawbacks of conventional hash tables, by allowing the embeddings of different $n$-gram orders in the LLM to share common latent structures?}

To answer this question, we propose the \textbf{T}ensorized E\textbf{ngram} (\textbf{TN-gram}), which models the whole $n$-gram space in a tensorized format using compact tensor networks in order to avoid hash collisions introduced by hash tables. More specifically, we represent each $n$-gram order in their Canonical Polyadic (CP) decomposition format and enforce the sharing of CP components across $n$-gram orders. This shared higher-order low-rank structure allows for the training of the embeddings of unseen $n$-grams. Additionally, benefiting from the super-compression capabilities of tensor networks, this drastically reduces the number of parameters needed to model the complete $n$-gram space from $\mathcal{O}(V^n \times d)$, which scales exponentially with the $n$-gram order, to $\mathcal{O}(nVR+dR)$, which scales linearly in $n$ under fixed tensor rank $R$.

In summary, our contributions are as follows:
\begin{itemize}
    \item We introduce \textbf{TN-gram}, a tensorized $n$-gram memory architecture that models the complete $n$-gram space efficiently using CP tensor network representations instead of hash tables.
    \item TN-gram utilizes shared CP factors across the embeddings of different $n$-gram orders, effectively utilizing the intrinsic nested structure of multiple $n$-gram orders, i.e., an $(n-1)$-gram is always contained in the higher-order $n$-grams.
    \item Comprehensive experiments demonstrate that TN-gram achieves comparable or better performance and requires less parameters at the same $n$-gram order compared to hashing-based solutions such as Engram \cite{deepseek_engram}. 
\end{itemize}

\section{Related Works}
\paragraph{Scaling embedding layers and explicit $n$-gram memories.} Recent work has explored scaling language model representations beyond standard token-level embeddings. Over-tokenized Transformer \cite{overtokenized_transformer} expands the embedding table to include higher-gram tokens and then combines their embeddings with token embeddings, showing that local multi-token units can provide useful inductive bias. Engram \cite{deepseek_engram} instead reformulates explicit $n$-gram embedding as conditional memory, extracting suffix $n$-grams ending at the current token for $n$-gram orders $2$ to $N$ through deterministic multi-head hashing. Each extracted $n$-gram is hashed into an order-specific memory table, from which memory vectors are retrieved and projected into Key/Value representations. A context-aware gate compares the Key with the current hidden state and filters the Value before adding it to the residual stream. Similar to \citet{deepseek_engram}, authors in \citet{meituan} also discovered the effectiveness in scaling the embedding layers compared to scaling the experts in the Mixture-of-Experts (MoE) layers. These works motivate embedding scaling as an orthogonal parameter-allocation direction once expert scaling in MoE reaches diminishing returns. However, despite their effectiveness, these approaches are limited by the hashed memory design. Finite lookup tables can cause hash collisions \cite{meituan} , and separate tables for each $n$-gram prevent related contexts from sharing parameters.

\paragraph{Tensor networks for LLMs.}
Tensor network methods are especially useful for achieving an efficient representation of large-scale neural networks due to their super-compressed capabilities. Early work showed that dense fully connected layers can be replaced with Tensor-Train tensor networks, significantly reducing parameters while preserving expressivity \cite{novikov2015tensorizing}. More recently, TensorLLM applies Tucker decomposition for post-training compression and denoising~\citep{Tensorllm}, while TensorGPT compresses embedding layers for deployment on edge-devices \cite{xu2023tensorgpt} using Tensor-Train decomposition. LoRETTA~\citep{loretta} and TeRA~\citep{Tera} use tensorized low-rank updates for parameter-efficient fine-tuning (PEFT), and KromHC~\citep{Kromhc} represents the tensorized residual hyper-connection weights with Kronecker-structured matrices to reduce parameter complexity. Other studies analyze how tensor-network rank affects the efficiency-expressivity tradeoff~\citep{zhou2025understanding} and use LLMs to guide tensor network structure and rank search~\citep{iacovides2025domain}.

\section{Preliminaries}

\paragraph{Tensor notations.}
A tensor is a multi-dimensional array that generalizes vectors and matrices to higher orders \cite{7038247}. A scalar, vector, matrix and tensor is denoted as $a, \mathbf{a}, \mathbf{A}, \mathcal{A}$, respectively. An order-$N$ tensor $\mathcal{T}\in\mathbb{R}^{I_1\times \cdots \times I_N}$ has $N$ modes, with $\mathcal{T}(i_1,\ldots,i_N)$ denoting the entry at index $(i_1,\ldots,i_N)$ \cite{TN1}. The mode-$n$ product multiplies a tensor $\mathcal{T}\in\mathbb{R}^{I_1\times\cdots\times I_N}$ with a matrix $\mathbf{A}\in\mathbb{R}^{J\times I_n}$ along its $n$-th mode, producing a tensor $\mathcal{T}\times_n \mathbf{A} \in \mathbb{R}^{I_1\times\cdots\times I_{n-1}\times J\times I_{n+1}\times\cdots\times I_N}.$ Its element-wise entry is given by 
\begin{equation}
\begin{aligned}
    &(\mathcal{T}\times_n \mathbf{A})(i_1,\ldots,i_{n-1},j,i_{n+1},\ldots,i_N) 
    \\&= \sum_{i_n=1}^{I_n} \mathcal{T}(i_1,\ldots,i_n,\ldots,i_N) \mathbf{A}(j,i_n). 
\end{aligned}
\end{equation}

\paragraph{Canonical Polyadic tensor network decomposition.} A fundamental multi-linear tensor network for representing large tensors is the Canonical Polyadic (CP) decomposition \cite{kolda_tensor}, which approximates $\mathcal{T}$ as a sum of $R$ rank-one components as follows
\begin{equation}
\mathcal{T} = \sum_{r=1}^{R} \mathbf{a}^{(1)}(r)\circ \mathbf{a}^{(2)}(r)\circ \cdots \circ \mathbf{a}^{(N)}(r),
\end{equation}
or in the element-wise format
\begin{equation}
\mathcal{T}(i_1,\ldots,i_N) = \sum_{r=1}^{R} \mathbf{A}^{(1)}(i_1 ,r)\mathbf{A}^{(2)}(i_2 ,r)\cdots \mathbf{A}^{(N)}(i_N ,r),
\end{equation}
where $\circ$ denotes the outer product, $R$ is the tensor rank, $\mathbf{A}^{(k)}\in\mathbb{R}^{I_k\times R}$ is the factor matrix for mode $k$, and $\mathbf{a}^{(k)}_{r}\in\mathbb{R}^{I_k}$ is the $r$-th column of $\mathbf{A}^{(k)}$. CP decomposition reduces the number of parameters from $\prod_{k=1}^{N} I_k$ in the original $\mathcal{T}$ to $\sum_{k=1}^{N} I_kR$, where $R \ll I_k$, achieving super-compression via representing the original tensor in a higher-order low-rank format.

\paragraph{The $n$-gram concept in Engram.}
Traditionally, $n$-gram language models represent local dependencies in discrete token sequences by making a finite-context Markov assumption, where the probability of the next token depends solely on the preceding $n-1$ tokens \cite{manning1999foundations}. While classical $n$-gram models \cite{bojanowski2017enriching, brants2007large, nguyen2024understanding} compute these probabilities through empirical corpus counts, recent efforts such as Engram \cite{deepseek_engram} modernize this concept by formalizing $n$-grams as a conditional memory module for LLMs. More specifically, Engram uses the local context to perform $\mathcal{O}(1)$ lookups \cite{pagnoni2025byte,tito2017hash,yu2025scaling} into an external embedding table. To circumvent the intractable parameter size of the complete combinatorial space, Engram adopts a multi-head hashing strategy. Given a maximum order $N$, Engram includes all context orders $n \in \{2, \ldots, N\}$. For each order $n$ and head $k \in \{1,\ldots,K\}$, a deterministic hash function maps the local context $x_{t,n}$ to an index, retrieving an embedding $\mathbf{e}_{t,n,k}$. The final memory vector $\mathbf{e}_t$ concatenates the retrieved embeddings across all $n$-gram orders and hash heads as follows
\begin{equation}
    \mathbf{e}_t = \mathop{||}_{n=2}^{N} \mathop{||}_{k=1}^{K} \mathbf{e}_{t,n,k},
\end{equation}
where $||$ denotes concatenation and $K$ is the number of hash heads. 

Unlike Over-tokenized Transformer \cite{overtokenized_transformer} which simply expands the embedding vocabulary beyond $1$-gram, Engram introduces a context-aware gating mechanism to filter the retrieved static memory using the hidden state $\mathbf{h}_t$ at the current sequence position $t$. It computes Key and Value vectors via linear projections, i.e., $\mathbf{k}_t = \mathbf{W}_K \mathbf{e}_t$ and $\mathbf{v}_t = \mathbf{W}_V \mathbf{e}_t$, and then computes a scalar gate:
\begin{equation}
    \alpha_t = \sigma \left( \frac{\RMSNorm(\mathbf{h}_t)^\top \RMSNorm(\mathbf{k}_t)}{\sqrt{d}} \right),
    \label{equ:context gating}
\end{equation}
where $\sigma(\cdot)$ is the sigmoid function and $d$ is the representation dimension. The gated output $\alpha_t \mathbf{v}_t$ dynamically modulates the static $n$-gram memory. Although this approach successfully offloads the reconstruction of local patterns from the Transformer layers to external lookups, its reliance on independent hash tables per $n$-gram order fundamentally prevents the sharing of latent structures between nested $n$-gram contexts.

\section{Methodology}
We first review how explicit $n$-gram contexts are extracted and used in Engram-style memory modules. We then introduce our proposed TN-gram, which replaces order-specific hashed lookup tables with a shared tensorized $n$-gram CP representation space. Finally, we describe how the resulting memory modules are integrated into Transformer blocks through context-aware gating and projection absorption. Figure~\ref{fig:architecture} provides an overview of this pipeline and compares the separately hashed memory design of Engram with our proposed shared CP-tensorized memory in TN-gram.

\subsection{Suffix $n$-gram Extraction}
At each position $t$, we extract suffix $n$-grams ending at the current token. Given $x_{1:T}=(x_1,\ldots,x_T)$ and maximum order $N$, the order-$n$ context is
$x_{t-n+1:t}=(x_{t-n+1},\ldots,x_t)$ for $n\in\{2,\ldots,N\}$, with left padding
when the context extends beyond the beginning of the sequence. For example, the
2-gram and 3-gram contexts are $(x_{t-1},x_t)$ and $(x_{t-2},x_{t-1},x_t)$.
Similar to \cite{deepseek_engram}, we consider $n\geq 2$ because 1-gram information is already captured by the standard token embedding. Each such context indexes an $n$-gram latent vector, whose compact parameterization is described next.

\subsection{Tensorizing $n$-gram Embeddings in TN-gram}
The collection of latent vectors indexed by the suffix contexts above can be viewed as an $n$-gram embedding table. For a single $n$-gram order, the original $n$-gram space is $\mathbb{R}^{V^n \times d}$, representing a conventional table that stores one $d$-dimensional vector for each of the $V^n$ possible contexts. Therefore, we can naturally tensorize it to an order-$(n+1)$ tensor $\in \mathbb{R}^{V \times V \times \dots \times V \times d}$ where the first $n$ modes index token positions and the final mode indexes the latent dimension. This tensor view provides a structure that flat hashed tables lack, i.e., suffix contexts of different orders are nested. For example, the $2$-gram, $(x_{t-1}, x_t)$, is the suffix of the $3$-gram, $(x_{t-2},x_{t-1},x_t)$, suggesting that their representations should share latent patterns rather than be learned independently. TN-gram exploits this structure by sharing the CP factors for the tensors of different orders $n \in \{2,\ldots,N\}$. 

\subsection{TN-gram Formulation}\label{lower_order_ngrams_via_full_gram}
\paragraph{Extracting lower-order $n$-grams from the largest-order $N$-gram.}
Let $N$ be the largest $n$-gram order that we wish to include. For any order $n \le N$, TN-gram represents the corresponding embedding tensor using the shared CP factors. More specifically, the embedding tensor for $N$-gram is defined as follows
\begin{equation}
\begin{aligned}
  \mathcal{T}^{(N)} =
  \mathcal{I}
  &\times_1 \mathbf{A}_1 \times_2
  \cdots
  \times_{N-n} \mathbf{A}_{N-n}
  \times_{N-n+1} \\
  &\mathbf{A}_{N-n+1}
  \cdots
  \times_N \mathbf{A}_N
  \times_{N+1} \mathbf{F},
\end{aligned}
\end{equation}
where the order-$(N+1)$ tensor, $\mathcal{T}^{(N)} \in \mathbb{R}^{V\times V \times \cdots \times d}$, is the $N$-gram embedding tensor, $\mathcal{I} \in \mathbb{R}^{R \times \cdots \times R}$ is a super-diagonal identity tensor, and $\times_i$ denotes the mode-$i$ product. The $N$ shared CP factors, $\mathbf{A}_1,\ldots,\mathbf{A}_N \in \mathbb{R}^{V \times R}$, are so-called token-position factor matrices ordered from oldest to newest token, while $\mathbf{F} \in \mathbb{R}^{d \times R}$ is the shared model dimension factor matrix, and $R$ is the tensor rank.

\begin{figure*}
  \centering
  \includegraphics[width=1\linewidth]{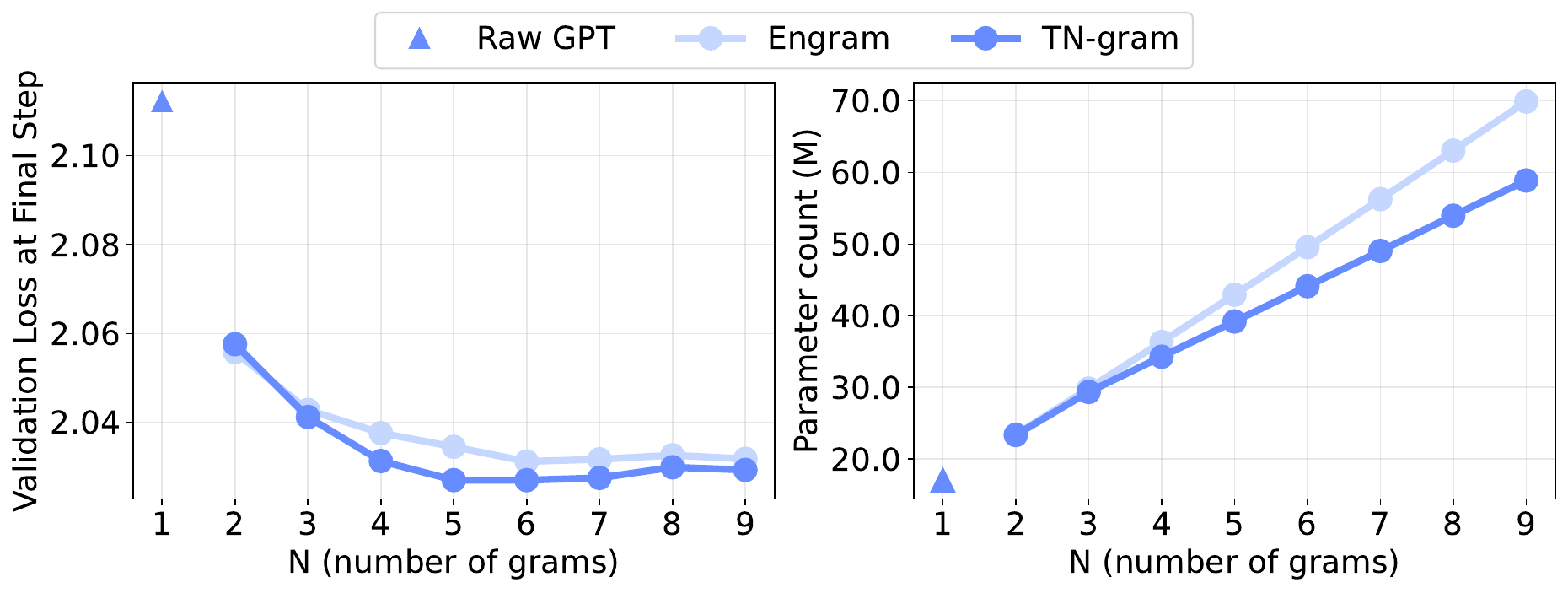}
  \caption{\textbf{Left}: Validation loss of the proposed TN-gram improves as $N$-gram order increases, with most of the gain achieved by 5-gram and little improvement thereafter. \textbf{Right}: Parameter count of both methods grows roughly linearly with $N$, with TN-gram using fewer parameters than the Engram baseline for the same $N$, with a CP tensor rank of $1200$.
  \label{fig:n_gram_order}
}
\end{figure*}

For a lower-order $n$-gram, since the lower-order $n$-gram context simply occupy the last $n$ token modes of the $N$-gram, we use shared vectors, $\hat{\mathbf{w}}_1,\ldots,\hat{\mathbf{w}}_{N-2} \in \mathbb{R}^R$, to obtain its embedding tensor
\begin{equation}
  \mathcal{T}^{(n)} = \mathcal{T}^{(N)}\times_1 \hat{\mathbf{w}}_1 \times_2
  \cdots
  \times_{N-n} \hat{\mathbf{w}}_{N-n}.
\end{equation}
Due to the multilinear property of tensor networks, notice that $\mathbf{A}_k$ can be absorbed into $\hat{\mathbf{w}}_{k}$ for $1\leq k \leq N-n$. In practice, we simply need to learn order-absorption vectors,$\mathbf{w}_{k}$, as $\mathbf{w}_{k} = \mathbf{A}_k \hat{\mathbf{w}}_{k}$. Therefore, the embedding tensor for $n$-gram can be efficiently written with the shared token-position factor matrices, $\mathbf{A}_1,\ldots,\mathbf{A}_N \in \mathbb{R}^{V \times R}$, and shared order-absorption matrices, $\mathbf{w}_1,\ldots,\mathbf{w}_{N-2} \in \mathbb{R}^R$, as follows 
\begin{equation}
\begin{aligned}
\mathcal{T}^{(n)} =
  \mathcal{I}
  &\times_1 \mathbf{w}_1 \times_2 
  \cdots
  \times_{N-n} \mathbf{w}_{N-n}
  \times_{N-n+1} \\
  &\mathbf{A}_{N-n+1}
  \cdots
  \times_N \mathbf{A}_N
  \times_{N+1} \mathbf{F},
\end{aligned}
\label{eq:ngram_embedding_formulation}
\end{equation}

To obtain the embedding for a specific order-$n$ context $(x_1,\ldots,x_n)$, since the CP core $\mathcal{I}$ is super-diagonal, the first $N$ mode-$n$ products in Eq. (\ref{eq:ngram_embedding_formulation}) reduce to an element-wise (Hadamard) product of the selected token factors and the order-absorption vectors, in the form
\begin{equation}
\begin{aligned}
    &\mathbf{b}_n (x_{N-n+1:N}) =
    \mathbf{w}_1 \odot \cdots \odot \mathbf{w}_{N-n}\\
    &\quad \quad \quad \odot\mathbf{A}_{N-n+1}(x_{N-n+1},:)\odot \cdots \odot
    \mathbf{A}_N(x_N,:),
\end{aligned}
\label{eq:elementwise_cp}
\end{equation}
where $\odot$ denotes the Hadamard product and $\mathbf{A}_i[x,:]$ selects the $x^{\text{th}}$ row of $\mathbf{A}_i$.

Under the standard CP-factorized parameterization, the embedding vector would have been obtained by multiplying the vector $\mathbf{b}_n (x_{N-n+1:N}) \in \mathbb{R}^R$ with $\mathbf{F}\in\mathbb{R}^{d\times R}$ as $\mathbf{h}^{(n)} = \mathbf{b}_n \mathbf{F}^\top.$ For numerical stability in practice, we first normalize and rescale the token-space representation as
\begin{equation}
    \mathbf{e}_{n} = s_n \mathrm{RMSNorm}\left( \mathbf{b}_n \right),
\label{eq:rms_scale}
\end{equation}
where $s_n=\exp(l_{n})$ is a learnable scalar for $n$-gram order $n \in \{2,\ldots,N\}$. RMSNorm stabilizes these vectors, while the learned scalar $s_n$ lets the model adaptively weigh $n$-gram contexts of different order. To obtain the embedding vector, we simply need to do $\mathbf{h}^{(n)}=\mathbf{e}_n\mathbf{F}^{\top}\in\mathbb{R}^d$. In Section \ref{sec:Absorbed}, we detail how TN-gram avoids materializing this multiplication by absorbing $\mathbf{F}$ into the subsequent Key/Value projections, allowing the memory module to operate directly on $\mathbf{e}_n$. We ablate these design choices comprehensively in Section~\ref{ablation_components}.

\paragraph{Efficient key and value projections by absorbing $\mathbf{F}$.} 
\label{sec:Absorbed}
We use the Engram-style context-aware gating mechanism~\cite{deepseek_engram} to filter $n$-gram memory using the current hidden state, $\mathbf{h}_t$. Given the token-space representations $\mathbf{e}_2, \ldots, \mathbf{e}_N$ from Eq. (\ref{eq:rms_scale}), a naive implementation would first project each order to model dimension using the CP model dimension factor $\mathbf{F}$, concatenate the resulting embeddings, and then apply Key/Value projections to obtain $\mathbf{d}_{\mathrm{concat}}=\left[ \mathbf{e}_{2}\mathbf{F}^\top \mid \cdots \mid \mathbf{e}_{N}\mathbf{F}^\top
\right]$. For the Key projection, let 
$\mathbf{W}_K= \begin{bmatrix}
    \mathbf{W}_{K,2} \\ 
    \vdots \\
    \mathbf{W}_{K,N}
\end{bmatrix} \in \mathbb{R}^{(N-1)d_{\mathrm{model}}\times d_{\mathrm{model}}},$ 
we have 
\begin{equation}
\begin{aligned}
    \mathbf{k}_t
    =
    \mathbf{d}_{\mathrm{concat}}\mathbf{W}_K
    &=
    \sum_{n=2}^{N}
    \left(\mathbf{e}_{n}\mathbf{F}^\top\right)
    \mathbf{W}_{K,n} \\
    &=
    \sum_{n=2}^{N}
    \mathbf{e}_{n}
    \left(\mathbf{F}^\top\mathbf{W}_{K,n}\right).
\end{aligned}
\end{equation}

Since both multiplications are linear, we can absorb $\mathbf{F}^{\top}\mathbf{W}_{K,n}$ into a single learned matrix $\mathbf{M}_{K,n}\in\mathbb{R}^{R\times d_{\mathrm{model}}}$, and do the same
for the Value projection. In practice, we concatenate the token-space vectors
directly, $\mathbf{e}_{\mathrm{concat}} = \left[\mathbf{e}_2 \mid \cdots \mid \mathbf{e}_N \right] \in \mathbb{R}^{1\times (N-1)R},$ and compute
\begin{equation}
    \mathbf{k}_t = \mathbf{e}_{\mathrm{concat}}\mathbf{M}_K,
    \qquad
    \mathbf{v}_t = \mathbf{e}_{\mathrm{concat}}\mathbf{M}_V,
\end{equation}
where $\mathbf{M}_K,\mathbf{M}_V\in\mathbb{R}^{(N-1)R\times d_{\mathrm{model}}}$. This avoids materializing intermediate matrix multiplications and removes the constraint that each order-specific Key/Value map must have the same model dimension factor $\mathbf{F}$. Finally, TN-gram follows the same Engram context-aware gating mechanism~\cite{deepseek_engram} as detailed in Eq. (\ref{equ:context gating}), where the Value vector, $\mathbf{v}_t$, is filtered by a scalar gate computed from the current hidden state and Key vector, $\mathbf{k}_t$, passed through a depthwise causal convolution, and added to the residual stream. The full gating and convolution equations are given in Appendix~\ref{app:gating}.

\begin{table*}[!ht]
\centering
\setlength{\tabcolsep}{1.3pt}
\caption{Comparison of Engram and TN-gram with raw GPT at 9 and 18 layers.}
\label{tab:loss}
\begin{tabular}{lcccccccc}
\toprule
Method 
& \multicolumn{4}{c}{$L=9,N=5$ $(R=1024)$} 
& \multicolumn{4}{c}{$L=18, N=5$ $(R=1800)$} \\
\cmidrule(lr){2-5} \cmidrule(lr){6-9}
& $\Delta$ Params $\downarrow$
& Train Loss $\downarrow$ 
& Val BPB $\downarrow$ 
& CORE $\uparrow$
& $\Delta$ Params $\downarrow$
& Train Loss $\downarrow$ 
& Val BPB $\downarrow$ 
& CORE $\uparrow$ \\
\midrule
Raw GPT
& -- & 2.209 & 1.251 & 0.06
& -- & 1.792 & 1.086 & 0.107 \\
Engram
& 26M & 2.142 & 1.209 & 0.072
& 60M & 1.771 & \textbf{1.070} &  0.115\\
\rowcolor{blue!10}
TN-gram \textbf{(Ours)}
& \textbf{19M} & \textbf{2.141} & \textbf{1.208} & \textbf{0.083}
& \textbf{48M} & \textbf{1.770} & 1.071 & \textbf{0.120}\\
\bottomrule
\end{tabular}
\label{tab:train_val_loss_core}
\end{table*}

\subsection{Sharing Latents across $n$-gram Embeddings}
The formulation of TN-gram uniquely ensures that the embeddings of a specific $n$-gram are constructed from the shared factors across different $n$-gram orders. We formalize this into Theorem \ref{theo:subspace_inclusion}, which states that the token-space extraction vectors of a $n-1$-gram is a linear combination of the token-space extraction vectors of a $n$-gram.

\begin{theorem}
    \textbf{Shared latents across $n$-grams embeddings.}
    Let $\mathbf{b}_{n}(x_{N-n+1:N})$ denote the unnormalized token-space extraction
    vector of an $n$-gram as defined in Eq.~(\ref{eq:elementwise_cp}), with tokens ordered from oldest to newest. For any $n \in \{3,\ldots,N\}$ where $N$-gram is the largest order, if $\mathbf{A}_{N-n+1}$ has full column rank, the token-space extraction vector of the shorter $(n-1)$-gram context, obtained by dropping the oldest token $x_{N-n+1}$, satisfies
    \begin{equation}
        \mathbf{b}_{n-1}(x_{N-n+2:N})
        =
        \sum_{x_{N-n+1} \in \mathcal{V}}
        c_{x_{N-n+1}}
        \mathbf{b}_{n}(x_{N-n+1:N}),
    \end{equation}
    where $c_{x_{N-n+1}} \in \mathbb{R}$ depends only on
    $\mathbf{w}_{N-n+1}$ and $\mathbf{A}_{N-n+1}$.
    \label{theo:subspace_inclusion}
\end{theorem}

\section{Experiments}\label{experiments}
To evaluate the performance of our proposed TN-gram method, we conducted experiments on LLM pretraining at two scales, i.e., 9 transformer layers and 18 transformer layers. Our experimental setup builds on Parameter-Golf \cite{param_golf}. All models were trained on data from FineWeb \cite{fineweb}. Our results show that TN-gram achieves performance on par or better than Engram across all evaluated $n$-gram orders, while requiring fewer trainable parameters.

\paragraph{Effect of $n$-gram Order.}
Figure~\ref{fig:n_gram_order} depicts how validation loss and parameter count vary with the maximum $n$-gram order $N$. Validation loss generally improves as $N$ increases, with most gains achieved by $N=5$ and only marginal improvements and sometimes degradations thereafter. 

An exception occurs at $N=2$, where TN-gram has slightly higher validation loss than Engram. This is expected, since the main advantage of TN-gram comes from sharing latent factors across nested $n$-gram orders, i.e., using lower-order suffixes to support higher-order contexts. When $N=2$, there is only one explicit $n$-gram order, which means the cross-order sharing mechanism is absent. In this case, TN-gram behaves as a CP-factorized bigram memory, whereas Engram can assign more flexible independent embeddings through its hashed lookup table. Moreover, $2$-gram contexts are much denser and better observed than higher-order $n$-grams, so Engram suffers less from the occurrence sparsity limitations that become more significant at larger $N$.

As $N$ increases, TN-gram can exploit its shared factors and order-absorption vectors more effectively, which leads to better validation performance. Since most of the performance gains are achieved by $N=5$, with only marginal improvements beyond that, we use $N=5$ for the main comparisons in Tables~\ref{tab:train_val_loss_core} - \ref{tab:lm_bigbench}.

\paragraph{Accuracy-Efficiency Tradeoff.} Table~\ref{tab:train_val_loss_core} compares Engram and TN-gram against a raw GPT baseline at two transformer depths. We report four metrics: (1) the number of additional parameters relative to the raw GPT; (2) training loss; (3) validation bits-per-byte (BPB); and (4) the CORE score on downstream tasks. BPB is computed as
$\mathcal{L}_{\text{BPB}}=\frac{\mathcal{L}_{\text{cross-entropy}}}{\ln(2)}\times \frac{\text{Total tokens}}{\text{Total bytes}}$, while CORE score aggregates downstream task accuracy (See more details in Appendix \ref{downstream_tasks}).

\begin{table*}[ht!]
\centering
\caption{Downstream performance comparison on multiple-choice tasks. Accuracy is reported. Best result per task is bolded.}
\label{tab:downstream-mc}
\setlength{\tabcolsep}{2.3pt}
\begin{tabular}{l l cccccccccccc}
\toprule
$L$ & Method & Avg & HS-ZS & ARC-E & ARC-C & COPA & CSQA & PIQA & OBQA & HS & LSAT-AR & BoolQ & LangID \\
\midrule
\multirow{3}{*}{9}
& Raw GPT & 33.48 & 30.00 & 30.89 & 19.45 & 52.00 & 20.72 & 57.40 & 20.40 & 29.31 & 23.48 & 59.17 & \textbf{25.51} \\
& Engram & 34.14 & 30.14 & \textbf{31.52} & \textbf{20.48} & \textbf{55.00} & 21.13 & 57.51 & 23.80 & \textbf{29.87} & 22.61 & 58.17 & 25.36 \\
\rowcolor{blue!10}
& TN-gram & \textbf{35.44} & \textbf{30.23} & 31.10 & 19.45 & 54.00 & \textbf{30.71} & \textbf{57.62} & \textbf{24.40} & 29.75 & \textbf{25.22} & \textbf{61.99} & 25.35 \\

\midrule
\multirow{3}{*}{18}
& Raw GPT & 36.85 & 34.57 & 34.05 & \textbf{22.70} & 58.00 & 26.78 & 62.02 & 22.80 & 33.99 & 24.35 & 60.55 & 25.50 \\
& Engram & 37.16 & \textbf{35.56} & 35.90 & 22.44 & \textbf{60.00} & 26.13 & 60.99 & 22.20 & \textbf{34.60} & 23.04 & \textbf{62.14} & \textbf{25.72} \\
\rowcolor{blue!10}
& TN-gram & \textbf{37.82} & 35.05 & \textbf{36.07} & 22.44 & \textbf{60.00} & \textbf{28.58} & \textbf{63.00} & \textbf{23.60} & 34.41 & \textbf{25.22} & 62.08 & 25.53 \\
\bottomrule
\end{tabular}
\label{tab:multiple_choice}
\end{table*}

\begin{table*}[ht!]
\centering
\caption{Downstream performance comparison on language-modeling and BIG-bench tasks. Accuracy is reported. Best result per task is bolded.}
\label{tab:downstream-lm}
\setlength{\tabcolsep}{2.5pt}
\begin{tabular}{llcccccccccc}
\toprule
$L$ & Method & Avg & BBH-QA & LAMB & BBH-Dyck & BBH-CS & BBH-Op & SQuAD & CoQA & WG & WGrande  \\
\midrule
\multirow{3}{*}{9}
& Raw GPT & 21.68 & 3.23 & 20.53 & 13.40 & 38.41 & \textbf{13.33} & 0.34 & 2.00 & 53.48 & 50.36 \\
& Engram & 22.43 & 4.46 & 22.94 & \textbf{13.90} & \textbf{38.94} & 11.90 & 0.45 & 5.45 & 53.48 & 50.36 \\
\rowcolor{blue!10}
& TN-gram & \textbf{22.72} & \textbf{6.29} & \textbf{23.31} & \textbf{13.90} & 37.42 & 12.38 & \textbf{0.79} & \textbf{5.93} & \textbf{53.85} & \textbf{50.59}\\

\midrule
\multirow{3}{*}{18}
& Raw GPT & 25.35 & 9.05 & 28.64 & \textbf{16.90} & 39.77 & 15.71 & 0.96 & 7.95 & 57.51 & 51.62 \\
& Engram & 26.05 & \textbf{13.99} & \textbf{32.04} & 14.80 & 40.98 & 12.38 & 0.92 & 8.18 & \textbf{60.07} & 51.07 \\
\rowcolor{blue!10}
& TN-gram & \textbf{26.29} & 12.32 & 31.83 & 15.20 & \textbf{41.36} & \textbf{16.19} & \textbf{1.35} & \textbf{8.92} & 56.41 & \textbf{53.04} \\
\bottomrule
\end{tabular}
\label{tab:lm_bigbench}
\end{table*}

For a fair comparison, we chose the tensor rank in TN-gram to be $R=1024$ and $R=1800$ for the 9- and 18-layer models, respectively, so that TN-gram has less number of additional parameters to Engram. Also, both Engram and TN-gram use a 5-gram configuration. Across both depths, TN-gram achieved the best CORE score while adding fewer parameters than Engram. For 9 layers, it also obtained the best validation BPB, reducing BPB from $1.251$ to $1.208$. For 18 layers, it matches Engram in validation BPB ($1.071$ vs. $1.070$) while improving CORE from $0.107$ to $0.120$. These results suggest that TN-gram provides a stronger accuracy-efficiency tradeoff, since its shared CP-factorized memory couples well-observed lower-order $n$-grams with sparser higher-order ones, allowing latent structure from frequent patterns to inform the representations of rarer long-context patterns.

\paragraph{Downstream Evaluation.}
Tables~\ref{tab:multiple_choice} and~\ref{tab:lm_bigbench} compare Raw GPT, Engram, and TN-gram on downstream tasks. Across both 9- and 18-layer settings, TN-gram consistently improves average downstream accuracy over the raw GPT baseline and Engram. On multiple-choice tasks, TN-gram achieves the best average accuracy, improving from $33.48$ to $35.44$ for 9 layers and from $36.85$ to $37.82$ for 18 layers. Similar gains were observed on language-modeling and BIG-bench tasks, where TN-gram obtains the highest average score at both depths. These results suggest that TN-gram improves both validation loss and downstream generalization.

\paragraph{Effect of Tensor Rank}

We next investigate whether additional explicit-memory capacity is more useful when allocated to independently hashed memories or to the shared tensorized latent space. For Engram, we scale capacity by increasing the number of hashing heads, which adds more order-specific lookup tables. For our TN-gram, we scale capacity by increasing the CP tensor rank $R$, which enlarges the number of shared latent components used to construct all $n$-gram orders.

\begin{figure}[ht!]
  \centering
  \includegraphics[width=1\linewidth]{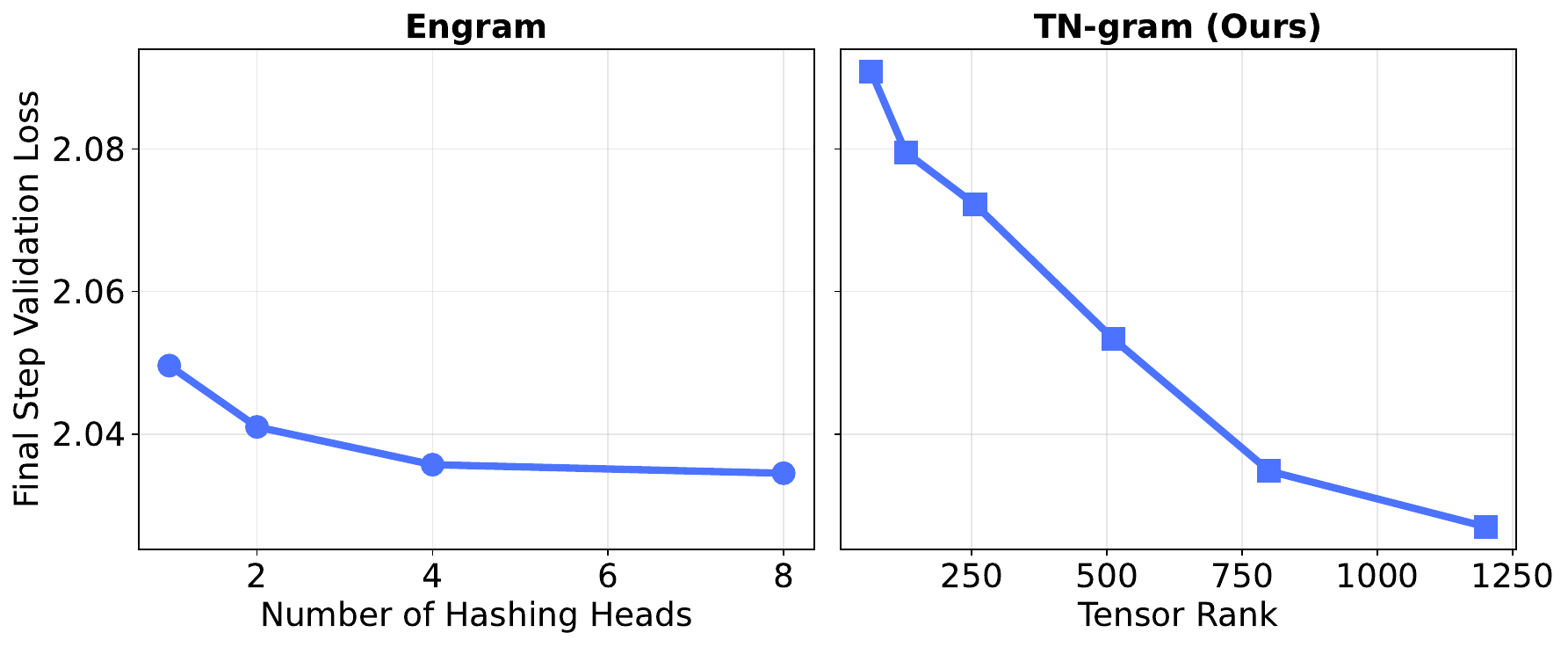}
  \caption{Scaling explicit $n$-gram memory capacity. \textbf{Left}: Increasing the number of hashing heads in Engram yields only limited validation-loss improvement. \textbf{Right}: increasing the CP tensor rank $R$ in TN-gram consistently improves validation loss.}
  \label{fig:hash_heads_cp_rank_scaling}
\end{figure}

Figure \ref{fig:hash_heads_cp_rank_scaling} shows a clear difference between two scaling directions. Increasing the number of Engram hashing heads yields modest improvements in validation loss, suggesting that duplicating independent hash memories has diminishing returns. This is consistent with Engram's design, in which additional heads increase the number of retrieved slots, but do not leverage the shared latent structure across nested $n$-gram contexts.

By contrast, increasing the CP rank in TN-gram produces a more consistent reduction in validation loss. A larger rank directly increases the expressivity of the shared tensorized $n$-gram space, while slightly increasing the parameter count. More importantly, higher rank improves the shared latent components from which different $n$-gram orders are composed, rather than simply adding more isolated memories.

\begin{figure*}
  \centering
  \includegraphics[width=0.85\linewidth]{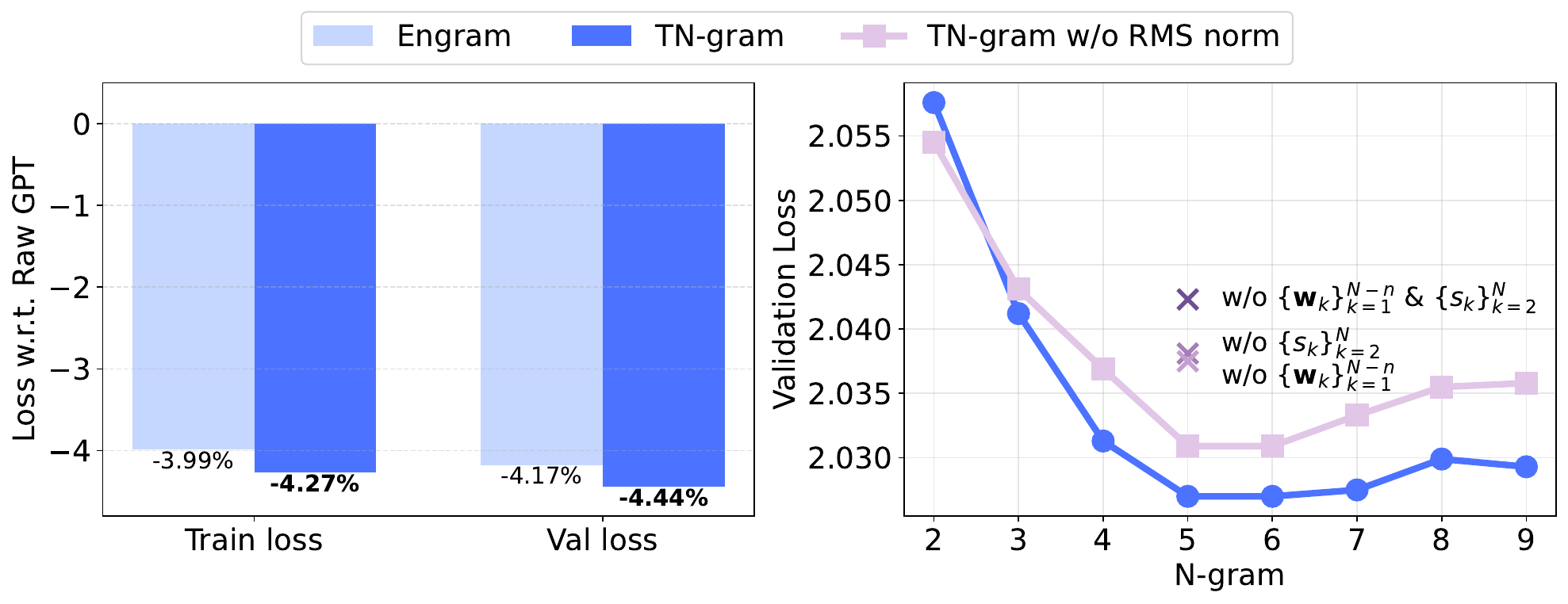}
  \caption{ Ablation Study. \textbf{Left}: Loss reduction relative to raw GPT for Engram and TN-gram for the vocabulary size 8192 with all $5$-grams. \textbf{Right}: Validation loss of $N$-gram order for TN-gram ablations, showing effects of removing RMS normalization, learnable scale $\{s_k\}_{k=2}^N$, and order-absorption vectors $\{\mathbf{w}_k\}_{k=1}^{N-n}$.}
  \label{fig:ablation}
\end{figure*}

\subsection{Ablation Study}

\paragraph{Effect of vocabulary size.}
Experiments in Tables \ref{tab:loss}, \ref{tab:multiple_choice}, \ref{tab:lm_bigbench}, and Figure \ref{fig:n_gram_order} use vocabulary size $1024$. To investigate whether the advantages of TN-gram scale to a larger vocabulary, Figure~\ref{fig:ablation} (left) compares Engram and TN-gram at vocabulary size $8192$, reporting loss reduction relative to the raw GPT baseline. Both methods improve training and validation loss, confirming that explicit $N$-gram modeling remains beneficial even at higher vocabulary sizes. TN-gram yields larger reductions than Engram, improving training loss by $4.27\%$ and validation loss by $4.44\%$, compared to $3.99\%$ and $4.17\%$ for Engram. This suggests that TN-gram provides a more effective parameterization of $N$-gram statistics even as the vocabulary size increases.

\paragraph{Ablation of TN-gram components.}\label{ablation_components}
Section~\ref{lower_order_ngrams_via_full_gram} introduces three components for stabilizing the learnt TN-gram representations: RMS normalization, learnable per-order scaling $\{s_k\}_{k=2}^N$, and learnable order-absorption vectors $\{\mathbf{w}_k\}_{k=1}^{N-n}$ for unspecified context modes. We ablate these components to assess their individual contributions. As shown in Figure~\ref{fig:ablation} (right), each removal increases validation loss, confirming that all three components are beneficial. In particular, removing the learnable scales has a larger effect than removing the order-absorption vectors alone, while removing both leads to the worst performance. The variant without RMS normalization is consistently worse than the full TN-gram across $n$-gram orders, highlighting the importance of normalization for stabilizing CP-based $n$-gram representations.

\section{Conclusion}

We have introduced the \textbf{T}ensorized E\textbf{ngram} (\textbf{TN-gram}), a compact explicit $n$-gram memory module for incorporating local multi-token context into Transformer blocks. Unlike Engram-style approaches that use separate hash tables for each $n$-gram order, TN-gram represents the tensorized $n$-gram spaces with shared CP factors. This design removes hash collision and, more importantly, enables lower- and higher-order $n$-gram contexts to share latent structure through common token-position factors and learnable order-absorption vectors. We anticipate that this work will open future avenues for explorations in tensorized $n$-gram memories.

\paragraph{Limitations and future work.}

TN-gram improves parameter efficiency, but its current implementation is not yet fully optimized for wall-clock speed. Compared with hash-based Engram, TN-gram performs additional structured operations such as factor gathering, Hadamard products, normalization, and scaling. These operations are simple and highly parallelizable, and we expect the remaining overhead to be substantially reduced with customized CUDA kernels. Another practical consideration is the choice of tensor rank $R$. In our experiments, $R$ is selected to match Engram's parameter budget, but the best rank may vary with model scale, vocabulary size, maximum $n$-gram order, and compute constraints. Since tensor rank selection is a challenging and active research area, developing adaptive rank-selection strategies for TN-gram is an open direction for future work.

\section*{Impact Statement}
The proposed TN-gram may have positive societal impact by improving parameter efficiency for explicit $n$-gram memories, potentially contributing to more efficient and sustainable language models. By compactly modeling recurring multi-token patterns, it may reduce the parameter overhead in the neural network layers.

\bibliography{example_paper}
\bibliographystyle{icml2026}

\newpage
\appendix
\onecolumn
\section{Proof of Theorem 1}\label{proof_theorem_1}
\setcounter{theorem}{0}
\begin{theorem}
    \textbf{Shared Latents Across $n$-grams.}
    Let $\mathbf{b}_{n}(x_{N-n+1:N})$ denote the unnormalized token-space extraction
    vector of an $n$-gram as defined in Eq.~(\ref{eq:elementwise_cp}), with tokens ordered from oldest to newest. For any $n \in \{3,\ldots,N\}$ where $N$-gram is the largest order, if $\mathbf{A}_{N-n+1}$ has full column rank, the token-space extraction vector of the shorter $(n-1)$-gram context, obtained by dropping the oldest token $x_{N-n+1}$, satisfies
    \begin{equation}
        \mathbf{b}_{n-1}(x_{N-n+2:N})
        =
        \sum_{x_{N-n+1} \in \mathcal{V}}
        c_{x_{N-n+1}}
        \mathbf{b}_{n}(x_{N-n+1:N}),
    \end{equation}
    where $c_{x_{N-n+1}} \in \mathbb{R}$ depends only on
    $\mathbf{w}_{N-n+1}$ and $\mathbf{A}_{N-n+1}$.
    \label{theo:subspace_inclusion_appendix}
\end{theorem}

\begin{proof}
Let $n \in \{3,\ldots,N\}$ and let $\mathbf{A} = \mathbf{A}_{N-n+1}\in\mathbb{R}^{V\times R}$ and $\mathbf{w} = \mathbf{w}_{N-n+1}\in\mathbb{R}^R$. We define the shared suffix vector as follows
\begin{equation}
    \mathbf{q}(x_{N-n+2:N})
    :=
    \mathbf{w}_1 \odot \cdots \odot \mathbf{w}_{N-n}
    \odot
    \mathbf{A}_{N-n+2}(x_{N-n+2},:)
    \odot \cdots \odot
    \mathbf{A}_{N}(x_N,:)
    \;\in \mathbb{R}^R,
\end{equation}
which collects every factor shared by orders $n$ and $n{-}1$. Based on Eq.~(\ref{eq:elementwise_cp}), we can write the two extraction vectors as
\begin{equation}
    \mathbf{b}_n(x_{N-n+1:N})
    =
    \mathbf{A}(x_{N-n+1},:) \odot \mathbf{q}(x_{N-n+2:N}), \ 
    \mathbf{b}_{n-1}(x_{N-n+2:N})
    =
    \mathbf{w} \odot \mathbf{q}(x_{N-n+2:N}).
\end{equation}
Since $\mathbf{A}$ has full column rank, its rows span $\mathbb{R}^R$. Therefore, there exist coefficients $\{c_{x_{N-n+1}}\}_{x_{N-n+1}\in\mathcal{V}}$ satisfying $\mathbf{w} = \sum_{x_{N-n+1}\in\mathcal{V}} c_{x_{N-n+1}}\,\mathbf{A}(x_{N-n+1},:)$. Writing out the Hadamard product gives
\begin{equation}
\begin{aligned}
        \mathbf{b}_{n-1}(x_{N-n+2:N})
    =
    \mathbf{w} \odot \mathbf{q}
    &=
    \biggl(\sum_{x_{N-n+1}\in\mathcal{V}} c_{x_{N-n+1}}\,\mathbf{A}(x_{N-n+1},:)\biggr) \odot \mathbf{q}\\
    &=
    \sum_{x_{N-n+1}\in\mathcal{V}} c_{x_{N-n+1}}\,
    \underbrace{\mathbf{A}(x_{N-n+1},:) \odot \mathbf{q}}_{\mathbf{b}_n(x_{N-n+1:N})},
\end{aligned}
\end{equation}
where $c_{x_{N-n+1}}$ depends only on $\mathbf{w}=\mathbf{w}_{N-n+1}$ and $\mathbf{A}=\mathbf{A}_{N-n+1}$.
\end{proof}

\section{Downstream Tasks}\label{downstream_tasks}

We evaluate downstream task performance using a broad suite of tasks covering commonsense reasoning, scientific question answering, reading comprehension, language modeling, symbolic reasoning, and factual knowledge. The multiple-choice group includes HellaSwag \cite{zellers2019hellaswag} zero-shot and 10-shot variants for grounded commonsense completion, ARC-Easy and ARC-Challenge \cite{clark2018think} for elementary and challenging scientific reasoning, COPA \cite{gordon-etal-2012-semeval} and CommonsenseQA \cite{talmor-etal-2019-commonsenseqa} for causal and commonsense inference, PIQA \cite{bisk2020piqa} for physical commonsense, OpenBookQA \cite{mihaylov2018can} for open-domain science facts, AGI-Eval LSAT-AR \cite{zhong2024agieval} for analytical reasoning, BoolQ \cite{clark-etal-2019-boolq} for yes/no question answering, and BigBench \cite{srivastava2023beyond} language identification. 

The language-modeling and generation-style group includes LAMBADA \cite{paperno-etal-2016-lambada} for long-range next-word prediction, SQuAD \cite{rajpurkar-etal-2016-squad} and CoQA \cite{reddy-etal-2019-coqa} for extractive and conversational question answering, Winograd \cite{levesque2012winograd} and Winogrande \cite{sakaguchi2021winogrande} for pronoun and commonsense coreference resolution, and BigBench \cite{srivastava2023beyond} tasks such as Wikidata QA, Dyck languages, CS algorithms, operators, and repeat-copy logic, which test factual recall, formal-language structure, algorithmic reasoning, symbolic operations, and sequence-copying ability. 

Following DataComp-LM \cite{li2024datacomp}, we report CORE as an aggregate downstream metric. CORE averages the centered accuracy across tasks, where each task score is normalized so that random guessing corresponds to 0 and perfect accuracy corresponds to 1. This makes results more comparable across tasks with different chance levels.

Together, these benchmarks test whether TN-gram’s validation-loss improvements translate into broader downstream generalization. We report both individual task accuracies and CORE, using CORE as a compact summary of overall downstream performance.

\section{Scaling $n$-gram Order in a 19-Layer Transformer}

To evaluate the scaling performance of our proposed TN-gram, we conducted experiments using 18-layer transformers with embedding dimension 1024, and trained it for 10,000 steps. The validation loss and the CORE score at the final step are shown in Figure \ref{fig:larger_model}.

\begin{figure}
  \centering
  \includegraphics[width=\linewidth]{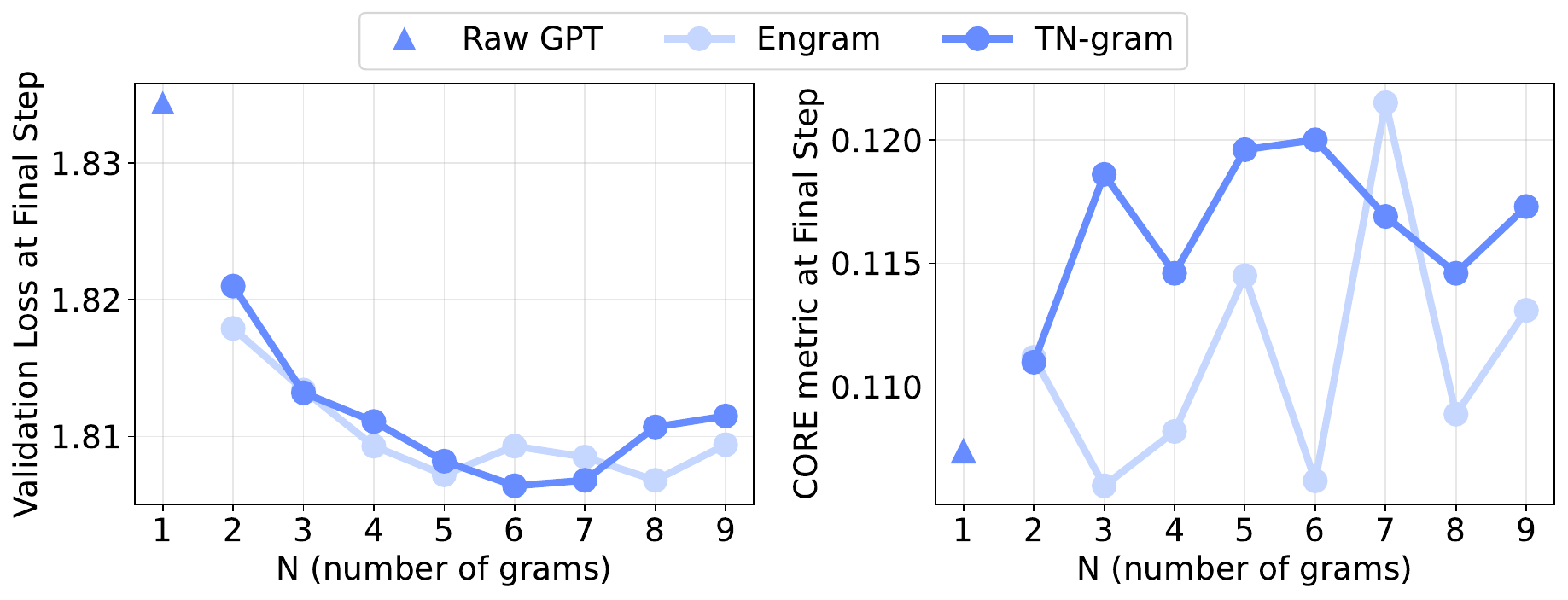}
  \caption{Final-step performance of 18-layer transformer models across $n$-gram orders $N$. \textbf{Left}: TN-gram yields the lowest validation loss than Engram and raw GPT baseline. The tensor rank in TN-gram is configured such that it always has less parameters than Engram. \textbf{Right}: TN-gram achieves higher CORE score, with strongest gain at larger $N$.}
  \label{fig:larger_model}
\end{figure}

\section{Context-Aware Gating}\label{app:gating}

This section details the Engram-style \cite{deepseek_engram} context-aware gating mechanism used to inject TN-gram memory into the Transformer. At position $t$, let $\mathbf{h}_t \in \mathbb{R}^{d_{\mathrm{model}}}$ denote the current hidden state, and let $\mathbf{k}_t,\mathbf{v}_t \in \mathbb{R}^{d_{\mathrm{model}}}$ denote the memory Key and Value vectors produced from the extracted $n$-gram representations. The gate compares the current context with the memory Key and suppresses memory values that are poorly aligned with the hidden state.

We first compute a normalized alignment score and transform it into a scalar gate:
\begin{equation}
    \alpha_t =
    \frac{
    \mathrm{RMSNorm}(\mathbf{h}_t)^\top
    \mathrm{RMSNorm}(\mathbf{k}_t)
    }{\sqrt{d_{\mathrm{model}}}},
    \qquad
    g_t =
    \sigma\!\left(
    \operatorname{sign}(\alpha_t)
    \sqrt{\max(|\alpha_t|,\epsilon)}
    \right),
\end{equation}
where $\sigma(\cdot)$ is the sigmoid function and $\epsilon=10^{-6}$. The signed square-root transformation reduces the magnitude of extreme dot products while preserving their sign. The gated memory is $\tilde{\mathbf{v}}_t=g_t\mathbf{v}_t$.

Finally, following Engram~\citep{deepseek_engram}, we apply a short depthwise causal convolution to the sequence of gated values before adding the memory output to the residual stream. Let $\tilde{\mathbf{v}}$ denote the sequence of gated values over time. With kernel size $w=3$, dilation $\delta=N$, and SiLU activation, the memory output added to the Transformer hidden state is
\begin{equation}
    \mathbf{y}_t =
    \tilde{\mathbf{v}}_t
    +
    \mathrm{SiLU}\!\left(
    \operatorname{Conv1D}
    \left(
    \mathrm{RMSNorm}(\tilde{\mathbf{v}})
    \right)
    \right)_t .
\end{equation}

\section{Hyperparameters}

Table~\ref{tab:shared-hparams} reports the shared training and model hyperparameters. 
Unless otherwise noted, all models are trained for 6000 steps on length-1024 sequences with a global batch size of 524{,}288 tokens. Depending on the size of the model, training was performed using either 4 $\times$ NVIDIA A100 (80GB) GPUs or a single L40S GPU (48GB). The backbone is a 9-layer Transformer with hidden size 512, 8 query heads, 4 KV heads, RoPE base 10{,}000, tied input/output embeddings, and a two-layer MLP with expansion factor 2. 
Engram and TN-gram insert Engram memory modules in layers 1 and 7.

\begin{table}[ht]
\centering
\caption{Shared hyperparameters.}
\label{tab:shared-hparams}
\begin{tabular}{ll}
\toprule
Name & Value \\
\midrule
Sequence length & 1024 \\
Global batch size & 524{,}288 tokens \\
Training steps & 6000 \\
Transformer blocks & 9 \\
Hidden dimension & 512 \\
Attention heads / KV heads & 8 / 4 \\
MLP expansion factor & 2 \\
RoPE base & 10{,}000 \\
Tied embeddings & Yes \\
Engram layers & 1, 7 \\
Random seed & 1337 \\
\bottomrule
\end{tabular}
\end{table}

We use Muon~\cite{jordan2024muon} for matrix-shaped Transformer parameters and Adam~\cite{kingma2014adam} for embeddings, scalar/control parameters, and the LM head. 
Muon uses learning rate 0.04, Nesterov momentum ramped from 0.85 to 0.95 over the first 500 steps, and 5 Newton--Schulz steps for orthogonalization. 
Adam uses $\beta_1=0.9$, $\beta_2=0.95$, and $\epsilon=10^{-8}$. 
For tied embeddings, the shared embedding table uses learning rate 0.05. Scalar/control parameters use learning rate 0.04.
Learning rates are constant until the final 1200 steps, then linearly decayed to zero.

\begin{table}[ht]
\centering
\caption{Optimizer hyperparameters.}
\label{tab:optim-hparams}
\begin{tabular}{ll}
\toprule
Name & Value \\
\midrule
Matrix optimizer & Muon \\
Matrix LR & 0.04 \\
Muon momentum & 0.85 $\rightarrow$ 0.95 over 500 steps \\
Adam $\beta_1,\beta_2,\epsilon$ & 0.9, 0.95, $10^{-8}$ \\
Tied embedding LR & 0.05 \\
Scalar/control LR & 0.04 \\
Warmdown steps & 1200 \\
\bottomrule
\end{tabular}
\end{table}

\end{document}